\documentclass[reqno]{amsart}

\usepackage{amssymb}
\usepackage{subfigure, graphicx}
\usepackage{hyperref}

\DeclareMathOperator{\sign}{sign}
\DeclareMathOperator{\Tr}{Tr}

\theoremstyle{plain}
\newtheorem{theorem}{Theorem}
\newtheorem{prop}{Proposition}

\theoremstyle{definition}
\newtheorem{problem}{Problem}

\theoremstyle{remark}
\newtheorem{remark}{Remark}

\begin{document}

\title[An Algorithmic Solution to the Five-Point Pose Problem]{An Algorithmic Solution to the Five-Point Pose Problem Based on the Cayley Representation of Rotations}

\author{E.V. Martyushev}

\date{February 2, 2013}

\keywords{Five-point pose problem, epipolar constraints, Cayley representation}

\address{South Ural State University, 76 Lenin Avenue, Chelyabinsk 454080, Russia}
\email{mev@susu.ac.ru}

\begin{abstract}
We give a new algorithmic solution to the well-known five-point relative pose problem. Our approach does not deal with the famous cubic constraint on an essential matrix. Instead, we use the Cayley representation of rotations in order to obtain a polynomial system from epipolar constraints. Solving that system, we directly get relative rotation and translation parameters of the cameras in terms of roots of a 10th degree polynomial.
\end{abstract}

\maketitle

\section{Introduction}

In the paper presented we give an algorithmic solution to the 5-point 2-view relative pose problem. It is formulated as follows.
\begin{problem}
\label{problem1}
We are given two calibrated pinhole cameras with centers $O_1$, $O_2$ and five points $Q_1, \ldots, Q_5$ lying in front of the cameras in 3-dimensional Euclidean space, see Figure~\ref{fig:Q5O2}. In every camera coordinate frame the directing vectors of $O_jQ_i$ are only known. The problem is in finding the relative position and orientation of the second camera with respect to the first one.
\end{problem}
\begin{figure}[ht]
\centering
\includegraphics[scale=0.3]{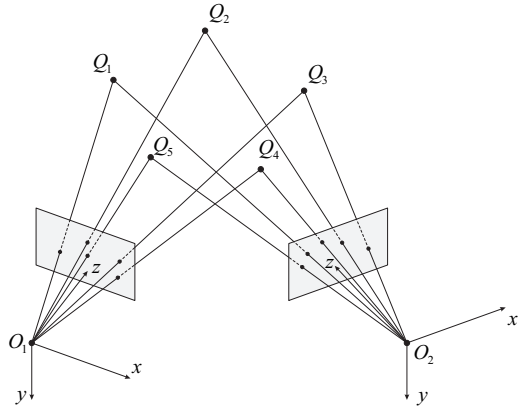}
\caption{To formulation of the five-point relative pose problem}\label{fig:Q5O2}
\end{figure}

The 5-point relative pose problem is a key to the 3d scene reconstruction problem, which is in turn used in many computer vision applications such as augmented reality, self-parking systems, robot path-planning, navigation, etc. It is well known that 5-point algorithms yield significantly better results in accuracy and reliability than 6-, 7- and 8-point algorithms. Moreover, for planar and near-planar scenes only 5-point method allows to get a robust solution without any additional modification of the algorithm.

Problem~\ref{problem1} was first shown by Kruppa~\cite{Kruppa} in 1913 to have at most eleven solutions. Using the methods of projective geometry, he proposed an algorithm for solving the problem, although it could not lead to a numerical implementation. Demazure~\cite{Demazure}, Faugeras and Maybank~\cite{FM}, Heyden and Sparr~\cite{HS} then sharpened Kruppa's result and proved that the exact number of solutions (including complex) is ten.

More efficient and practical solution has been presented by Philip~\cite{Philip} in 1996. His method requires to solve a 13th degree polynomial. In 2004 Nist\'{e}r~\cite{Nister} improved Philip's algorithm and expressed a solution in terms of a real root of 10th degree polynomial. Afterwards, there were presented many modifications of that algorithm simplifying its implementation~\cite{LH} or making it more numerically stable~\cite{KBP, SEN}.

In this paper we give yet another algorithmic solution to the problem using the well-known Cayley representation of rotation matrices~\cite{Cayley}. Our approach  does not mix rotation and translation parameters of an essential matrix and nevertheless allows one to express a solution in terms of a root of 10th degree univariate polynomial. Experiments on synthetic data show that the method is comparable in accuracy with the existing five-point solvers.

The rest of the paper is organized as follows. In Section~\ref{sec:description} we describe in detail our algorithm. In Section~\ref{sec:experiments} we make a comparison of our algorithm with the original Nist\'{e}r solver~\cite{Nister} on synthetic data. Section~\ref{sec:discussion} concludes.

\subsection{Notation}

We use $\mathbf a, \mathbf b, \ldots$ for column vectors, and $\mathbf A, \mathbf B, \ldots$ for matrices. For a matrix $\mathbf A$, the entries are $A_{ij}$, the transpose is $\mathbf A^{\mathrm T}$, the trace is $\Tr(\mathbf A)$, and the determinant is $\det(\mathbf A)$. For two vectors $\mathbf a$ and $\mathbf b$, the vector product is $\mathbf a\times \mathbf b$, and the scalar product is $\mathbf a^{\mathrm T} \mathbf b$. For a vector $\mathbf a$, the notation $[\mathbf a]_\times$ stands for a skew-symmetric matrix such that $[\mathbf a]_\times \mathbf b = \mathbf a \times \mathbf b$ for any vector~$\mathbf b$.

We use $\mathbf I$ for identical matrix and $\mathbf 0$ for zero matrix or vector, $\|\cdot\|$ for the Frobenius norm.

\section{Description of the algorithm}
\label{sec:description}

\subsection{Initial data transformation}
\label{ssec:transform}
Initial data for our algorithm are the homogeneous coordinates $x_{ji}$, $y_{ji}$, $z_{ji}$ of points $Q_i$ in the coordinate frame of $j$th camera, $j = 1, 2$, $i = 1, \ldots, 5$ (see Figure~\ref{fig:Q5O2}).

Without loss of generality we can set $x_{j1} = y_{j1} = x_{j2} = 0$ for $j = 1, 2$. The numerically stable way of doing this is as follows. We combine the initial data into two $3\times 5$ matrices
\begin{equation}
\label{eq:matrixAj}
\mathbf A_j = \begin{bmatrix}x_{j1} & \ldots & x_{j5} \\ y_{j1} & \ldots & y_{j5} \\ z_{j1} & \ldots & z_{j5} \end{bmatrix},
\end{equation}
and compute the matrices
\begin{equation}
\label{eq:rotations}
\mathbf A''_j = \mathbf H_{j2} \mathbf A'_j = \mathbf H_{j2} \mathbf H_{j1} \mathbf A_j,
\end{equation}
where $\mathbf H_{j1}$ and $\mathbf H_{j2}$ are the Householder matrices zeroing $x_{j1}$, $y_{j1}$ and~$x_{j2}$ respectively. The corresponding Householder vectors are
\[
\mathbf h_{j1} = \begin{bmatrix}x_{j1} \\ y_{j1} \\ z_{j1} + \sign(z_{j1})\sqrt{x_{j1}^2 + y_{j1}^2 + z_{j1}^2}\end{bmatrix}, \quad \mathbf h_{j2} = \begin{bmatrix}x_{j2}' \\ y_{j2}' + \sign(y_{j2}')\sqrt{x_{j2}'^2 + y_{j2}'^2} \\ 0\end{bmatrix}.
\]

We will see that transformation~\eqref{eq:rotations}, being quite simple, noticeably simplifies our further computations. In particular, this will allow us to easily convert the resulting 20th degree univariate polynomial~\eqref{eq:W} to the 10th degree polynomial~\eqref{eq:tildeW}.

\subsection{Epipolar constraints and essential matrix}

We first recall some definitions from multiview geometry, see~\cite{Faugeras, HZ, Maybank} for details. A \textit{pinhole camera} is a triple $(O, \pi, \mathbf P)$, where $\pi$ is an image plane, $\mathbf P$ is a central projection of points in 3-dimensional Euclidean space onto~$\pi$, and $O$ is a camera center (center of projection~$\mathbf P$). The \textit{focal length} is the distance between $O$ and~$\pi$, the orthogonal projection of~$O$ onto~$\pi$ is called the \textit{principal point}. A pinhole camera is called \textit{calibrated} if all its intrinsic parameters (such as focal length and principal point's coordinates) are known.

Let there be given two calibrated pinhole cameras $(O_j, \pi_j, \mathbf P_j)$, $j = 1, 2$. Without loss of generality we can set $\mathbf P_1 = \begin{bmatrix}\mathbf I & \textbf{0}\end{bmatrix}$, $\mathbf P_2 = \begin{bmatrix}\mathbf R & \mathbf t\end{bmatrix}$, where $\mathbf R \in \mathrm{SO}(3)$ is the \textit{rotation matrix} and $\mathbf t = \begin{bmatrix}t_1 & t_2 & t_3\end{bmatrix}^{\mathrm T}$ is the \textit{translation vector} normalized so that $\|\mathbf t\| = 1$.

The well-known \textit{epipolar constraints}~\cite{HZ} on $\mathbf R$ and $\mathbf t$ read:
\begin{equation}
\label{eq:epipolar}
\begin{bmatrix}
x_{2i} & y_{2i} & z_{2i}
\end{bmatrix} \mathbf E \begin{bmatrix} x_{1i} \\ y_{1i} \\ z_{1i} \end{bmatrix} = 0, \qquad i = 1, \ldots, 5,
\end{equation}
where $\mathbf E = [\mathbf t]_\times \mathbf R$ is called the \textit{essential matrix}.

\subsection{Ten fourth degree polynomials}
Our approach is based on the following well-known result.
\begin{theorem}[\cite{Cayley}]
\label{thm:cayley}
If a matrix $\mathbf R \in \mathrm{SO}(3)$ is not a rotation through the angle~$\pi + 2\pi k$, $k \in \mathbb Z$, about certain axis, then $\mathbf R$ can be represented as
\begin{equation}
\label{eq:cayley}
\mathbf R = \left(\mathbf{I} - \begin{bmatrix}u\\ v\\ w\end{bmatrix}_\times\right)\left(\mathbf{I} + \begin{bmatrix}u\\ v\\ w\end{bmatrix}_\times\right)^{-1},
\end{equation}
where $u, v, w \in \mathbb R$.
\end{theorem}

Let $\mathbf R$ be represented by~\eqref{eq:cayley} and $\mathbf E(u, v, w, \mathbf t) = [\mathbf t]_\times \mathbf R$ be an essential matrix.
\begin{prop}
\label{prop:invtrans}
If
\begin{equation}
\label{eq:invtrans}
\begin{split}
u' &= \frac{-t_1 - v t_3 + w t_2}{\delta},\\
v' &= \frac{-t_2 - w t_1 + u t_3}{\delta},\\
w' &= \frac{-t_3 - u t_2 + v t_1}{\delta},
\end{split}
\end{equation}
where $\delta = ut_1 + vt_2 + wt_3$, then $\mathbf E(u', v', w', \mathbf t) = -\mathbf E(u, v, w, \mathbf t)$.
\end{prop}

\begin{proof}
Consider a matrix $\mathbf R' = -\mathbf H_{\mathbf t} \mathbf R \in \mathrm{SO}(3)$, where the Householder matrix $\mathbf H_{\mathbf t} = \mathbf I - 2\mathbf t \mathbf t^{\mathrm T}$. Then, $\mathbf E' = [\mathbf t]_\times \mathbf R' = -\mathbf E$. By a straightforward computation, the equation $\mathbf R'(u', v', w') = -\mathbf H_{\mathbf t} \mathbf R(u, v, w)$ has a unique solution~\eqref{eq:invtrans}.
\end{proof}

Since epipolar constraints~\eqref{eq:epipolar} are linear and homogeneous in $\mathbf t$, we can rewrite them as
\begin{equation}
\label{eq:epipolar2}
\mathbf S\, \mathbf t = \mathbf 0,
\end{equation}
where the $i$th row of $5\times 3$ matrix $\mathbf S$ is
\[
\begin{bmatrix} x_{1i} & y_{1i} & z_{1i}\end{bmatrix} \mathbf R^{\mathrm T} \begin{bmatrix} x_{2i} \\ y_{2i} \\ z_{2i}\end{bmatrix}_\times.
\]
Now we represent rotation $\mathbf R$ in form~\eqref{eq:cayley} and take the determinants of all $3\times 3$ submatrices of matrix $\mathbf S$. This yields ten polynomial equations:
\begin{multline}
\label{eq:systemfi}
f_i = [0]u^4 + [0]u^3v + [0]u^2v^2 + [0]uv^3 + [0]v^4 + [1]u^3 + [1]u^2v \\ +[1]uv^2 + [1]v^3 + [2]u^2 + [2]uv + [2]v^2 + [3]u + [3]v + [4] = 0,
\end{multline}
where $i = 1, \ldots, 10$, $[n]$ means a polynomial of degree~$n$ in the variable~$w$, $[0]$ is a constant.

\begin{remark}
Actually, the determinants of $3\times 3$ submatrices of~$\mathbf S$ give the following expressions:
\[
\frac{F_i}{\Delta^3},
\]
where $\Delta = 1 + u^2 + v^2 + w^2$ and $F_i$ is a polynomial of 6th total degree. However, one can verify that $F_i$ is factorized as $F_i = f_i \Delta$ and the coefficients of~$f_i$ are easily deduced from the coefficients of~$F_i$.
\end{remark}

\subsection{Tenth degree univariate polynomial}
\label{ssec:10thdegree}

Let us rewrite system~\eqref{eq:systemfi} in form
\begin{equation}
\label{eq:sistemBm}
\mathbf B\mathbf m = \mathbf 0,
\end{equation}
where $\mathbf B$ is a $10\times 35$ coefficient matrix and
\[
\mathbf m = \begin{bmatrix}u^4 & u^3v & u^3w & \ldots & v & w & 1 \end{bmatrix}^{\mathrm T}
\]
is a monomial vector.

We expand system~\eqref{eq:sistemBm} with 20 more polynomials $uf_i$, $vf_i$ for $i = 1, \ldots, 5$, and $wf_i$ for $i = 1, \ldots, 10$. Thus we get
\begin{equation}
\label{eq:sistemBm2}
\mathbf B' \begin{bmatrix}\mathbf m' \\ \mathbf m\end{bmatrix} = \mathbf 0,
\end{equation}
where $\mathbf B'$ is a new $30\times 50$ coefficient matrix and
\begin{multline*}
\mathbf m' = [u^4w, u^3vw, u^3w^2, u^2v^2w, u^2vw^2, u^2w^3, \\uv^3w, uv^2w^2, uvw^3, uw^4, v^4w, v^3w^2, v^2w^3, vw^4, w^5]^{\mathrm T}
\end{multline*}
is the five-degree monomial vector. It is clear that system~\eqref{eq:sistemBm2} is equivalent to~\eqref{eq:sistemBm}.

We rearrange columns of matrix $\mathbf B'$ and perform Gauss-Jordan elimination with partial pivoting on it. Then the last six rows of the resulting matrix can be represented in form
\begin{center}
\smallskip\begin{tabular}{|c|cccccccccc|}
\hline
 & $u^3w^2$ & $u^3w$ & $u^3$ & $v^3w^2$ & $v^3w$ & $v^3$ & $uv$ & $u$ & $v$ & $1$\\\hline
$g_1$ & $1$ &  &  &  &  &  & $[3]$ & $[4]$ & $[4]$ & $[5]$\\
$g_2$ &  & $1$ &  &  &  &  & $[3]$ & $[4]$ & $[4]$ & $[5]$\\
$g_3$ &  &  & $1$ &  &  &  & $[3]$ & $[4]$ & $[4]$ & $[5]$\\
$g_4$ &  &  &  & $1$ &  &  & $[3]$ & $[4]$ & $[4]$ & $[5]$\\
$g_5$ &  &  &  &  & $1$ &  & $[3]$ & $[4]$ & $[4]$ & $[5]$\\
$g_6$ &  &  &  &  &  & $1$ & $[3]$ & $[4]$ & $[4]$ & $[5]$\\\hline
\end{tabular},\smallskip
\end{center}
where empty spaces are occupied by zeroes. Also, we have omitted first 28 zero columns. From the corresponding six polynomials $g_1, \ldots, g_6$ we obtain the following four polynomials
\begin{equation}
\begin{bmatrix}h_1 \\ h_2 \\ h_3 \\ h_4\end{bmatrix} \equiv
\begin{bmatrix}g_1 \\ g_2 \\ g_4 \\ g_5\end{bmatrix} -
w \begin{bmatrix}g_2 \\ g_3 \\ g_5 \\ g_6\end{bmatrix} = \mathbf C(w) \begin{bmatrix}uv \\ u \\ v \\ 1 \end{bmatrix} = \mathbf 0,
\end{equation}
where matrix $\mathbf C(w)$ can be represented as
\begin{equation}
\label{eq:matrixC}
\mathbf C(w) = \begin{bmatrix}[4] & [5] & [5] & [6] \\ [4] & [5] & [5] & [6] \\ [4] & [5] & [5] & [6] \\ [4] & [5] & [5] & [6]\end{bmatrix}.
\end{equation}

\begin{remark}
Since we use only six last rows of matrix~$\mathbf B'$, there is no need to perform a ``complete'' Gauss-Jordan elimination on matrix~$\mathbf B'$. For the first 24 rows of~$\mathbf B'$ only lower triangular entries should be zeroed.
\end{remark}

Denote by $\mathcal W = \det \mathbf C(w)$. In general,  it is a 20th degree polynomial in~$w$.
\begin{prop}
\label{prop:symm}
Polynomial $\mathcal W$ has a special symmetric form:
\begin{equation}
\label{eq:W}
\mathcal W = \sum\limits_{k=0}^{10} p_k \left[w^{10+k} + (-w)^{10-k}\right],
\end{equation}
where $p_k \in \mathbb R$.
\end{prop}

\begin{proof}
Due to the conditions $x_{j1}=y_{j1}=0$, we have $E_{33} = 0$. As a consequence,
\[
t_2 = t_1 \frac{R_{23}}{R_{13}} = t_1 \frac{vw + u}{uw - v}.
\]
Substituting this into the last identity in~\eqref{eq:invtrans}, we get $w' = -w^{-1}$. Thus, if $w_i$ is a root of~$\mathcal W$, then so is $-w_i^{-1}$. It follows that
\[
\mathcal W = p_{10}\prod\limits_{i=1}^{10} (w - w_i)(w + w_i^{-1}) = \sum\limits_{k=0}^{10} p_k \left[w^{10+k} + (-w)^{10-k}\right].
\]
\end{proof}

Substituting $\tilde w = w - w^{-1}$, we transform $\mathcal W$ to a 10th degree polynomial
\begin{equation}
\label{eq:tildeW}
\tilde{\mathcal W} = \sum\limits_{k=0}^{10}\tilde p_k \tilde w^k,
\end{equation}
where $\tilde p_k$ can be deduced using the formula $\tilde w^k =  \sum\limits_{i=0}^k (-1)^i \binom{k}{i} w^{2i-k}$. The result reads
\begin{equation}
\label{eq:tildep}
\tilde p_k = \sideset{}{'}\sum\limits_{i=k}^{10} \frac{i}{k} \binom{\frac{i+k}{2}-1}{\frac{i-k}{2}} p_i,
\end{equation}
where the primed sum is taken over all $i$ from $k$ to 10 such that $i - k \bmod 2  = 0$. Note that in case $k = 0$ the r.h.s. of~\eqref{eq:tildep} becomes $\sideset{}{'}\sum\limits_{i=0}^{10} 2p_i$.

\subsection{Structure recovery}
A complex root of~$\tilde{\mathcal W}$ leads to a complex root of~$\mathcal W$ and by~\eqref{eq:cayley} to complex rotation matrix having no geometric interpretation. Hence only real roots of~$\tilde{\mathcal W}$ must be treated.

Real roots of $\tilde{\mathcal W}$ can be efficiently found first using Sturm sequences~\cite{HM} for isolating and then Ridders' method~\cite{Press} for polishing. Then we can recover the second camera matrix applying the following algorithm.

Let $\tilde w_0$ be a real root of~$\tilde{\mathcal W}$. First we find the value
\[
w_0 = \tilde w_0/2 + \sign(\tilde w_0)\sqrt{(\tilde w_0/2)^2 + 1},
\]
which is a root of $\mathcal W$ subject to $|w_0| \geq 1$. After that, we obtain the $u$- and $v$-components of the solution by applying Gaussian elimination with partial pivoting on matrix~$\mathbf C(w_0)$ in~\eqref{eq:matrixC}.

Then we find the entries of $\mathbf R$ by~\eqref{eq:cayley}. Given $\mathbf R$, the translation vector $\mathbf t$ can be found by performing Gaussian elimination with partial pivoting on matrix $\mathbf S(u_0, v_0, w_0)$ in~\eqref{eq:epipolar2}. Here we have also taken into account the normalization constraint $\|\mathbf t\| = 1$.

Let $\mathbf H_{\mathbf t} = \mathbf I - 2\mathbf t \mathbf t^{\mathrm T}$ and $\mathbf R' = -\mathbf H_{\mathbf t}\mathbf R$. It is well-known~\cite{HZ, Nister} that there are four possibilities for the second camera matrix: $\mathbf P_A = \begin{bmatrix}\mathbf R & \mathbf t\end{bmatrix}$, $\mathbf P_B = \begin{bmatrix}\mathbf R & -\mathbf t\end{bmatrix}$, $\mathbf P_C = \begin{bmatrix}\mathbf R' & \mathbf t\end{bmatrix}$ and $\mathbf P_D = \begin{bmatrix}\mathbf R' & -\mathbf t\end{bmatrix}$. The only of these matrices is correct, all others correspond to unfeasible configurations.

The true second camera matrix $\mathbf P_2$ can be derived from the so-called \emph{cheirality constraint} saying that all the scene points must be in front of the cameras. In particular, this is valid for the first scene point~$Q_1$. Denote by
\begin{equation}
\label{eq:c1c2}
c_1 = -\frac{t_1}{R_{13}} = -\frac{t_2}{R_{23}}, \qquad c_2 = c_1R_{33} + t_3.
\end{equation}
Then,
\begin{itemize}
\item
if $c_1>0$ and $c_2>0$, then $\mathbf P_2 = \mathbf P_A$;
\item
else if $c_1<0$ and $c_2<0$, then $\mathbf P_2 = \mathbf P_B$;
\item
else if $c'_1>0$ and $c'_2>0$, then $\mathbf P_2 = \mathbf P_C$;
\item
else $\mathbf P_2 = \mathbf P_D$.
\end{itemize}
Here the value $c'_1$ and $c'_2$ are computed in the same manner as $c_1$ and $c_2$ in~\eqref{eq:c1c2} with $\mathbf R$ being replaced by~$\mathbf R'$.

Finally, the initial second camera matrix is given by
\[
\mathbf P_2^{\textit{ini}} = (\mathbf H_{22}\mathbf H_{21})^{\mathrm T}\mathbf P_2 \begin{bmatrix}\mathbf H_{12}\mathbf H_{11} & \mathbf 0 \\ \mathbf 0 & 1\end{bmatrix},
\]
where the Householder matrices $\mathbf H_{j1}$ and $\mathbf H_{j2}$ are defined in Subsection~\ref{ssec:transform}.

\section{Experiments on synthetic data}
\label{sec:experiments}

\begin{figure}[t]
\centering
\subfigure[Default conditions. The median error is $1.56\times 10^{-13}$ for Nister and $2.94\times 10^{-10}$ for New5pt]
{\includegraphics[scale=0.3]{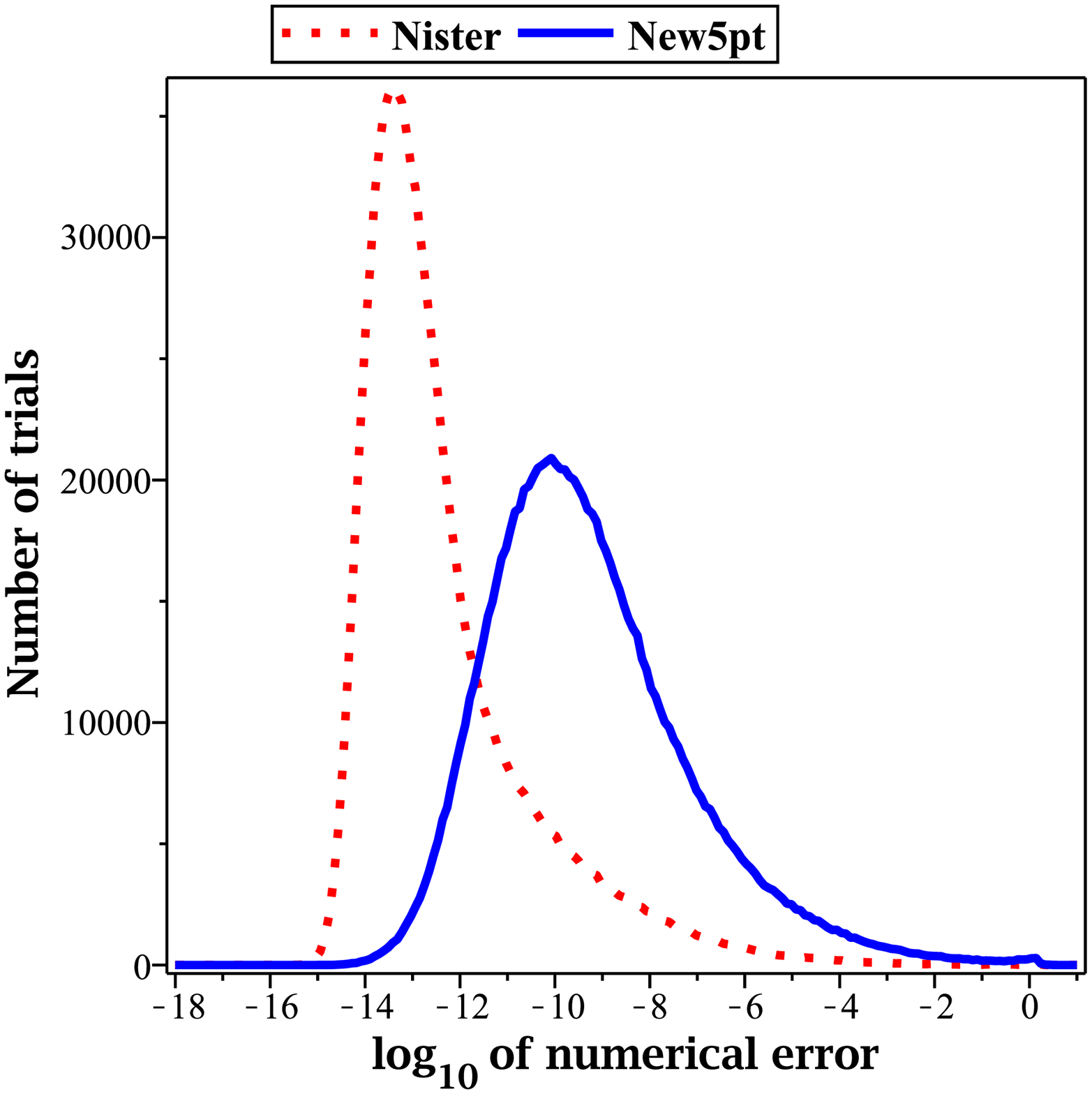}\label{fig:default}} \qquad
\subfigure[Planar scene and forward motion. The median error is $1.52\times 10^{-2}$ for Nister and $7.17\times 10^{-3}$ for New5pt]
{\includegraphics[scale=0.3]{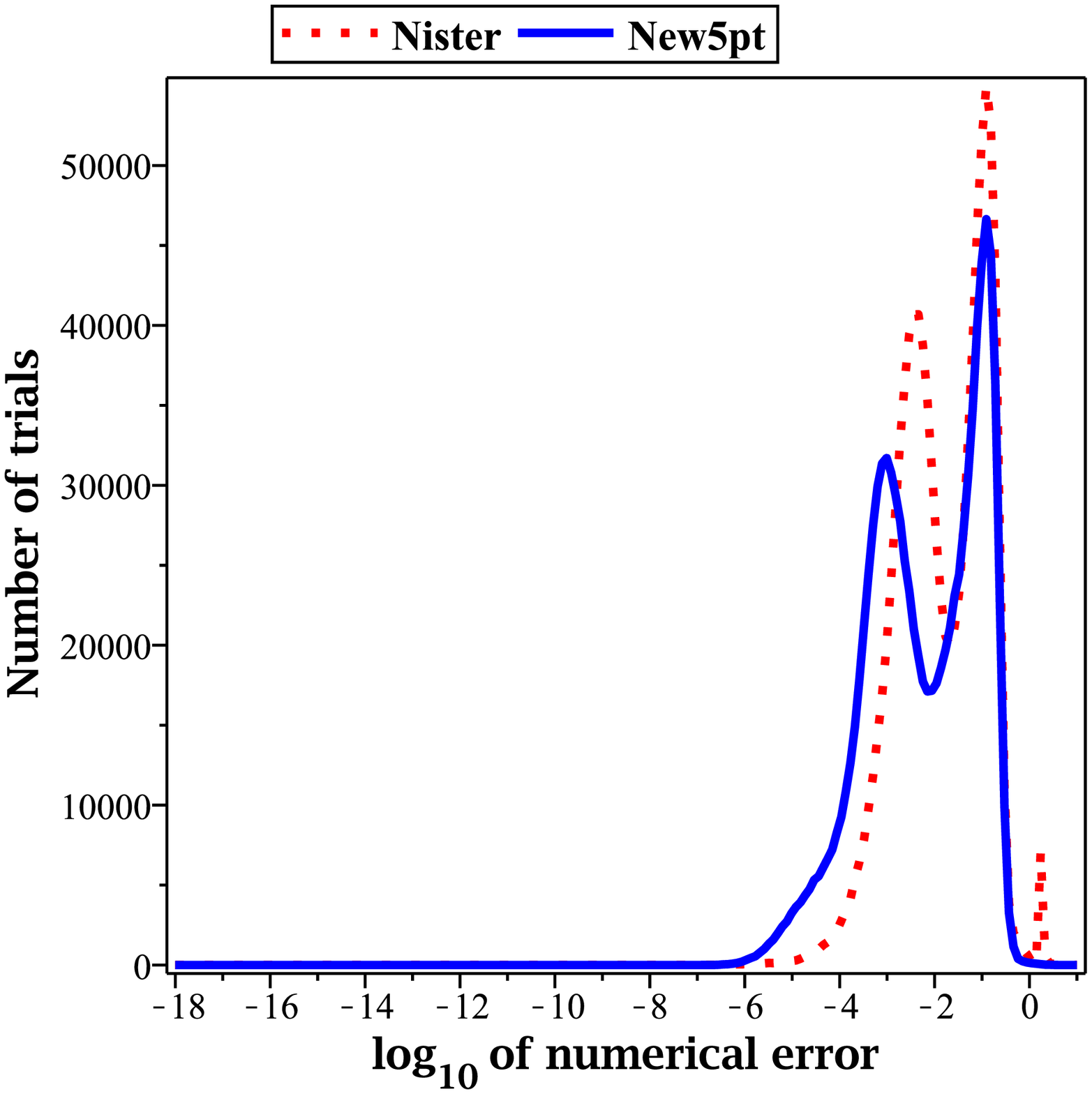}\label{fig:planar}}
\caption{Numerical error distribution}
\label{fig:errors}
\end{figure}

In this section we compare our algorithm with the original 5-point solver by Nist\'{e}r~\cite{Nister} on synthetic data. The C/C++ implementations of both algorithms have been written. All computations are performed in double precision. Synthetic data setup is the same as in~\cite{Nister}:
\begin{center}
\bigskip\begin{tabular}{|c|c|}
\hline
Distance to the scene & 1\\\hline
Scene depth & 0.5\\\hline
Baseline length & 0.1\\\hline
Image dimensions & $352 \times 288$ \\\hline
Field of view & 45 degrees\\\hline
\end{tabular}\bigskip
\end{center}

The \textit{numerical error} is defined by
\begin{equation}
\varepsilon = \|\bar{\mathbf P}_2 - \mathbf P_2\|,
\end{equation}
where $\bar{\mathbf P}_2$ is the ground truth second camera matrix.

The numerical error distributions are reported in Figure~\ref{fig:errors}. The total number of trials is $10^6$ in each experiment. We have compared the algorithms first in case of default conditions (Figure~\ref{fig:default}) and second in the most problematic case in sense of numerical stability --- planar scene and forward motion (Figure~\ref{fig:planar}).

In Figure~\ref{fig:transl_rot} we demonstrate the behaviour of the algorithms under increasing image noise. We add the Gaussian noise with a standard deviation varying from 0 to 1 pixel in a $352 \times 288$ image. One sees that in presence of noise the results of both algorithms are almost coincident.

\begin{figure}[t]
\centering
\subfigure[]
{\includegraphics[scale=0.3]{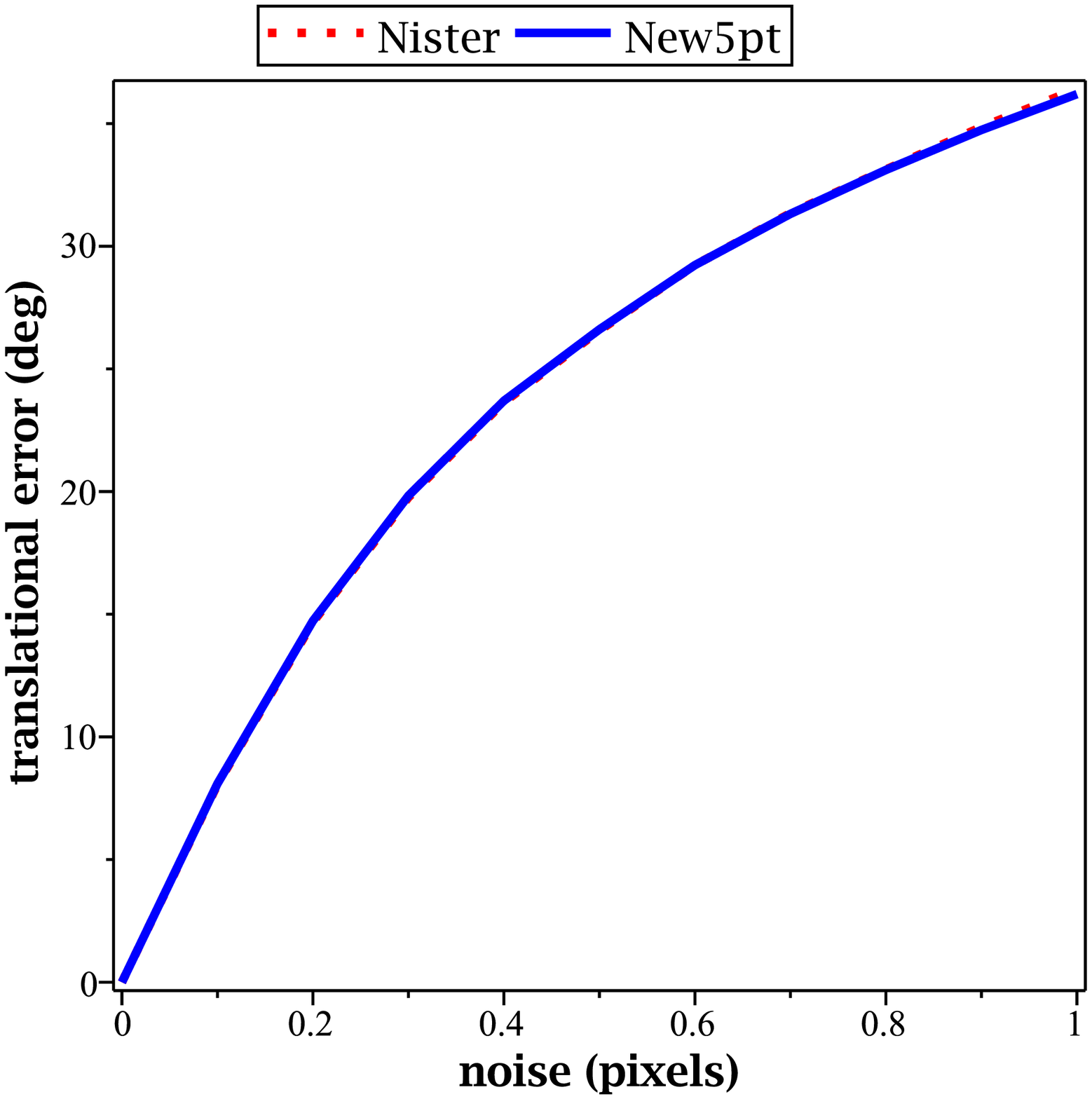}\label{fig:transl}}
\subfigure[]
{\includegraphics[scale=0.3]{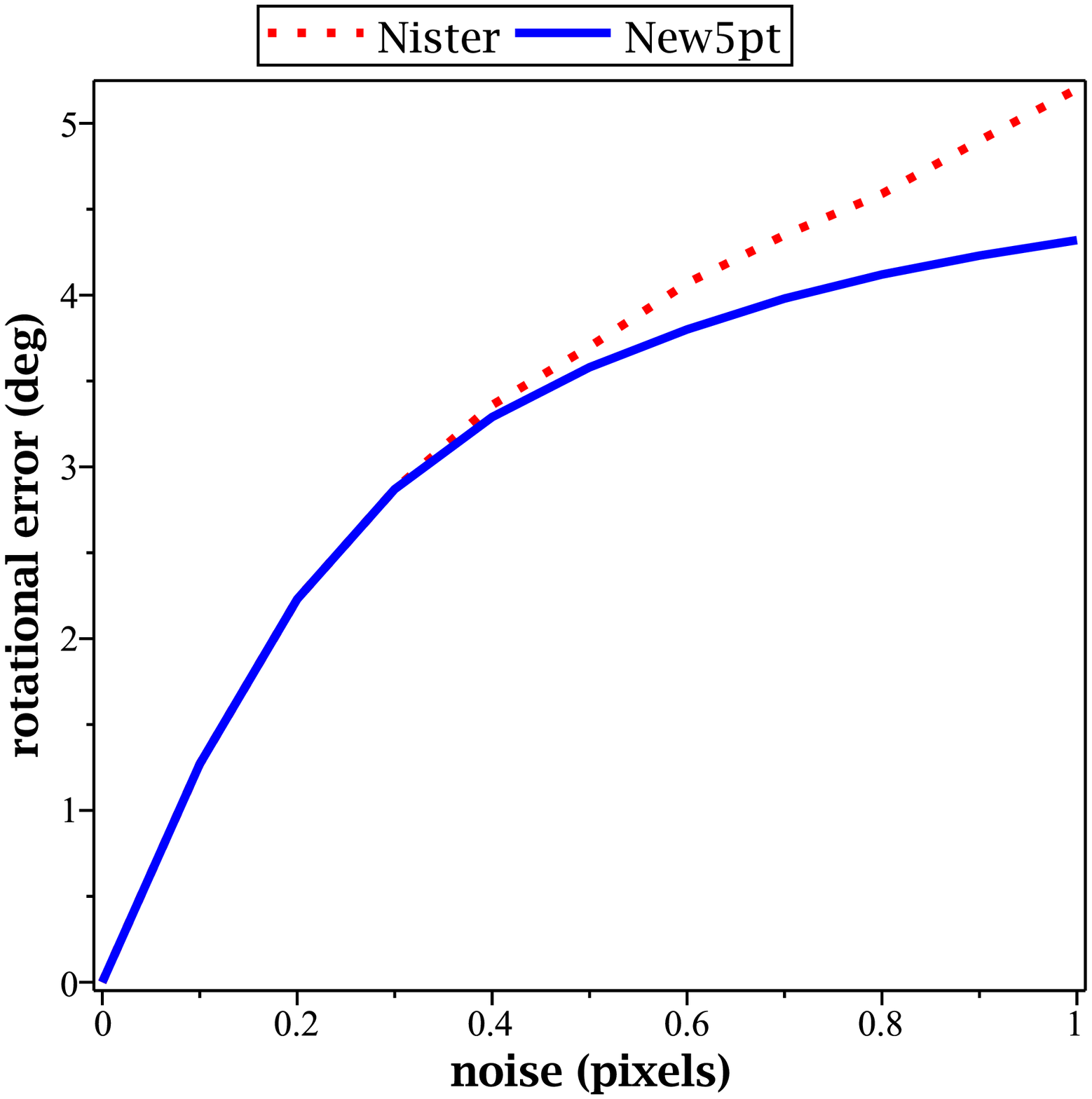}\label{fig:rot}} \qquad
\caption{Translational (left) and rotational (right) errors relative to Gaussian noise. Default conditions. Each point is a median of $10^6$ trials}
\label{fig:transl_rot}
\end{figure}

\section{Discussion of results}
\label{sec:discussion}
A new algorithm for the 5-point relative pose problem is presented. A computation on synthetic data confirms that it is robust enough. In whole, it is a good alternative to the existing five-point solvers. Its major advantage is that it yields a direct structure recovery, i.e. a reconstruction without computing an essential matrix. Such approach is more flexible when we are given some additional information on the camera rotations and/or translations. For instance, if the Euler angles $(\varphi, \theta, \psi)$, representing matrix $\mathbf R$, are known to lie in some limits, then so is the variable
\begin{equation}
\tilde w = -2\cot(\varphi + \psi).
\end{equation}
This allows one to discard some roots of the 10th degree polynomial~$\tilde{\mathcal W}$ at once without structure recovery step.

\bibliographystyle{amsplain}

\end{document}